\documentclass[twoside]{article}

\usepackage[preprint]{aistats2027}
\usepackage[round]{natbib}

\usepackage{algorithm}
\usepackage{algorithmic}

\usepackage{hyperref}

\usepackage{amsthm}

\usepackage{mathtools} 
\usepackage{amssymb}

\usepackage[capitalize,noabbrev]{cleveref}

\usepackage{booktabs}
\usepackage{makecell}
\usepackage{siunitx}

\theoremstyle{plain}
\newtheorem{theorem}{Theorem}[section]

\newtheorem{lemma}[theorem]{Lemma}

\theoremstyle{definition}

\newtheorem{assumption}[theorem]{Assumption}
\theoremstyle{remark}
\newtheorem{remark}[theorem]{Remark}

\crefname{theorem}{Theorem}{Theorems}
\Crefname{theorem}{Theorem}{Theorems}
\crefname{proposition}{Proposition}{Propositions}
\Crefname{proposition}{Proposition}{Propositions}
\crefname{lemma}{Lemma}{Lemmas}
\Crefname{lemma}{Lemma}{Lemmas}
\crefname{corollary}{Corollary}{Corollaries}
\Crefname{corollary}{Corollary}{Corollaries}
\crefname{definition}{Definition}{Definitions}
\Crefname{definition}{Definition}{Definitions}
\crefname{assumption}{Assumption}{Assumptions}
\Crefname{assumption}{Assumption}{Assumptions}
\crefname{remark}{Remark}{Remarks}
\Crefname{remark}{Remark}{Remarks}

\usepackage{notation}

\usepackage{enumitem}

\usepackage{newfloat}
\usepackage{listings}

\begin{document}

%

%

\twocolumn[

\aistatstitle{Learning Shortest Paths with Generative Flow Networks}

\aistatsauthor{ Nikita Morozov \And Ian Maksimov \And  Daniil Tiapkin \And Sergey Samsonov}

\aistatsaddress{ {\small HSE University} \And  {\small HSE University} \And {\small CMAP, CNRS, École polytechnique} \And  {\small HSE University} }



%
]

\begin{abstract}
  In this paper, we present a novel learning framework for finding shortest paths in graphs utilizing Generative Flow Networks (GFlowNets). First, we examine theoretical properties of GFlowNets in non-acyclic environments in relation to shortest paths. We prove that, if the total flow is minimized, forward and backward policies traverse the environment graph exclusively along shortest paths between the initial and terminal states. Building on this result, we reduce shortest-path finding in an arbitrary unweighted graph to finding a non-acyclic GFlowNet with minimal total flow, and approximate it in practice by training with flow regularization. We experimentally demonstrate the performance of our method in pathfinding in permutation environments and in solving Rubik's Cubes of various sizes. For the latter problems, our approach shows competitive results with state-of-the-art machine learning approaches designed specifically for this task in terms of the solution length, while requiring smaller search budget at test-time.
\end{abstract}

\section{Introduction}

Finding shortest paths in large discrete graphs is a fundamental problem in artificial intelligence, underlying applications in planning, routing, robotics, and combinatorial optimization~\citep{hart1968formal, gallo1988shortest, sutton1998reinforcement, schrijver2003combinatorial, madkour2017survey, karur2021survey}. Classical methods such as Dijkstra’s algorithm and A* provide complete and often optimal solutions when the graph can be explored and a suitable heuristic is available~\citep{hart1968formal}. However, designing accurate heuristics is often challenging in high-dimensional spaces, and many domains of practical interest induce state spaces so large that even storing a meaningful fraction of the graph is infeasible. In particular, this is the case for the combinatorial puzzles and planning problems whose transition structure forms a Cayley graph~\citep{cooperman1990applications}, where the number of reachable configurations may grow factorially with problem size. As a result, modern work has focused on learning-based approaches that approximate distance-to-goal functions and combine them with heuristic search at test time. Notable examples include deep reinforcement learning methods for solving the Rubik’s Cube~\citep{mcaleer2018solving, agostinelli2019solving} and more recent advances on general permutation puzzles based on learning to estimate random walk distances and using the estimates in beam search~\citep{chervov2025rubik}. 

Generative Flow Networks (GFlowNets), on the other hand, have emerged as a probabilistic framework for learning to sample compositional objects proportionally to a given reward function~\citep{bengio2021flow}, representing a synthesis of variational inference and reinforcement learning paradigms~\citep{malkin2022gflownets, zimmermann2022variational, tiapkin2024generative, deleu2024discrete}. GFlowNets define forward and backward policies over transitions in a directed acyclic graph, and enforce consistency between them via various objectives~\citep{malkin2022trajectory, bengio2023gflownet, madan2023learning}. These models have found success in different areas within discrete domains, including molecule generation~\citep{koziarski2024rgfn, cretu2025synflownet}, combinatorial optimization~\citep{zhang2023solving, kim2025ant}, causal discovery~\citep{manta2023gflownets, atanackovic2024dyngfn}, Bayesian structure learning~\citep{deleu2022bayesian, deleu2023joint}, and Bayesian inference in large language models~\citep{hu2023amortizing}.

Importantly for pathfinding, many environments of interest are cyclic: actions can be undone, and trajectories may revisit states. Recent works~\citep{brunswic2024theory, morozov2025revisiting} extend GFlowNets beyond acyclic generation graphs to non-acyclic environments and highlight that, in this setting, the expected trajectory length becomes a key quantity controlling sampling efficiency and avoiding pathological cycling behavior. However, the structural implications of minimizing this quantity have not been fully analyzed in these works.

In this paper, we establish a previously unexplored theoretical connection between minimizing expected trajectory length in non-acyclic GFlowNets and finding shortest paths. Based on this connection, we show how GFlowNets can be utilized as a probabilistic learning framework for solving general pathfinding problems. Our contributions can be summarized as follows:

\vspace{-0.25cm}
\begin{itemize}[itemsep=-2pt,leftmargin=8pt] 
    \item We prove that when the expected trajectory length is minimized, GFlowNet policies traverse the environment graph only along shortest paths between initial and terminal states, assigning zero probability to all non-shortest paths. The true minimizers are therefore excluded by the assumption $\mathcal{P}_B(s \mid s') > 0$ in prior theory \citep{morozov2025revisiting}, which we also address in our work. 
    \item Building on this characterization, we propose a constructive reduction from shortest-path problems  in unweighted graphs with a fixed goal state to training a non-acyclic GFlowNet that minimizes its expected trajectory length. Unlike heuristic-based approaches that learn value functions to guide search~\citep{agostinelli2019solving, chervov2025rubik}, our method directly aims to learn a policy that recovers shortest paths.
    \item We present a training algorithm based on a regularized variant of the trajectory balance objective~\citep{malkin2022trajectory}, and provide experimental validation on a synthetic permutation puzzle (Swap), as well as 2x2x2, 3×3×3 and 4x4x4 Rubik’s Cubes. We combine the learned policy with beam search and compare against a state-of-the-art approach for pathfinding in Rubik’s Cubes~\citep{chervov2025rubik}, demonstrating competitive solution lengths at substantially smaller search budgets.
\end{itemize}
\vspace{-0.1cm}
Source code: \href{https://github.com/GreatDrake/gfn-pathfinding}{github.com/GreatDrake/gfn-pathfinding}
\vspace{-0.2cm}

\section{Background}\label{sec:background}
\vspace{-0.1cm}

\subsection{Non-Acyclic GFlowNets}\label{sec:gflow_background}
\vspace{-0.1cm}

\citet{brunswic2024theory} and \citet{morozov2025revisiting} generalize the theoretical construction of GFlowNets \citep{bengio2023gflownet} to environments that may contain cycles. We follow the theory and notation of \citet{morozov2025revisiting}, which is summarized below.

The sequential decision-making process is described by a directed graph \(\cG = (\mathcal{S}, \mathcal{E})\), where \(\mathcal{S}\) is a finite state space and \(\mathcal{E} \subseteq \mathcal{S} \times \mathcal{S}\) is a finite set of edges (or transitions). The following assumption is made with respect to the structure of $\cG$:

\begin{assumption}
\label{graph_assumption}
    (1) There is a special \textit{initial state} $s_0$ with no incoming edges and a special \textit{sink state} $s_f$ with no outgoing edges; (2) For any state $s \in \cS$, there exists a path from $s_0$ to $s$ and from $s$ to $s_f$.
\end{assumption}

Let $\cT$ be a set of all finite-length trajectories $\tau = \left(s_0 \to s_1 \to \ldots \to s_{n_{\tau}} \to s_f\right)$ from $s_0$ to $s_f$, where we use $n_{\tau}$ to denote the length of the trajectory $\tau$. We use a convention $s_{n_{\tau} + 1} = s_f$. We say that $\tau$ terminates in a state $s$ if its last transition is $s \to s_f$. The edges of the form $(s \to s_f)$ are called terminating transitions, and the states $s$ that have an outgoing edge into $s_f$ are called \textit{terminal states}. The set of terminal states is denoted by $\cX$, and the probability distribution of interest $\cR(x) / \cZ$ is defined on it, where $\cR(x) > 0$ is called GFlowNet reward and $\cZ = \sum_{x \in \cX} \cR(x)$ is an unknown normalizing constant. 

GFlowNets operate with a pair of policies, where a \textit{forward policy} $\PF(s' \mid s)$ is a distribution over children of each state, and a \textit{backward policy} $\PB(s \mid s')$ is a distribution over parents of each state. During training, GFlowNets aim to find a pair of policies such that their induced distributions over trajectories coincide:
\begin{equation}
\label{eq:tb}
     \cP(\tau) = \prod_{t=0}^{n_\tau} \PF \left(s_{t+1} \mid s_{t}\right) = \prod_{t=0}^{n_\tau} \PB \left(s_{t} \mid s_{t+1}\right)\,.
\end{equation}
The main goal is to have such $\cP$ that for any $x \in \cX$, probability that $\tau \sim \cP$ terminates in $x$ coincides with $\cR(x) / \cZ$. This property is called the \textit{reward matching condition}, and w.r.t $\PB$ it can be simply written as \begin{equation}
\label{eq:rm}
\PB(x \mid s_f) = \cR(x) / \cZ \;\; \forall x \in \cX.
\end{equation}
If both~\eqref{eq:tb} and~\eqref{eq:rm} are satisfied, $\PF$ can be used to sample terminal states from the reward distribution (see \citet{bengio2023gflownet} and \citet{morozov2025revisiting} for more details).

\citet{brunswic2024theory} and \citet{morozov2025revisiting} show that the main difference of non-acyclic GFlowNets from their standard acyclic counterparts is the fact that the trajectory length $n_\tau$ is generally unbounded. Since $\E[n_\tau]$ exactly represents the average number of forward policy steps a GFlowNet takes to produce a sample, it also signifies the practical efficiency of the sampler. Thus, it is natural to formulate the problem of training a non-acyclic GFlowNet that minimizes $\E[n_\tau]$. \citet{morozov2025revisiting} define state flows through the expected number of times a certain state is visited
\begin{equation}
\label{eq:flow}
\cF(s) =  \cZ \cdot \;\E_{\tau \sim \cP}\left[ \sum_{t = 0}^{n_{\tau}+1} \ind\{s_t = s\}\right].
\end{equation}
Then, the expected trajectory length coincides with the normalized total flow $\E[n_\tau] = \tfrac{1}{\cZ} \sum_{s \in \cS \setminus \{s_0, s_f\}} \cF(s)$, so learning a model with the smallest $\E[n_\tau]$ is equivalent to learning a model with the smallest total flow. A simple practical approach proposed in \citet{brunswic2024theory} and further explored in \citet{morozov2025revisiting} involves adding a regularizer $\lambda \cF_\theta(s)$ to the utilized loss, where $\lambda$ is a regularization coefficient. E.g., detailed balance~\citep{bengio2023gflownet} loss with state flow regularization is defined as 
$
\mathcal{L}(s \to s') = \left(\log \frac{\cF_{\theta}(s) \PF(s' \mid s, \theta)}{\cF_{\theta}(s')\PB(s \mid s', \theta)} \right)^2 + \lambda \cF_\theta(s)\eqsp,
$
where reward matching is enforced by substituting $\cF_{\theta}(s_f) \cdot \PB(x  \mid s_f, \theta) = \cR(x)$.

\vspace{-0.15cm}
\subsection{Learning-based Pathfinding}
The problem of pathfinding on Cayley graphs has been of particular interest to the machine learning community~\citep{mcaleer2018solving, agostinelli2019solving, takano2023selfsupervision, chervov2025cayleypy, chervov2025rubik}. The most famous instance of this problem is solving Rubik's Cube. Here, all possible cube configurations are represented as vertices, and edges correspond to moves between them (turning a side of the cube), with the main goal consisting of finding the shortest paths to the solved configuration. Many other existing puzzles, games, and riddles can be described using Cayley graphs, where gameplay corresponds to finding a path towards a predefined goal state~\citep{mulholland2016permutation}.

Existing approaches generally follow a similar idea: train a neural network that evaluates how far a state is from the goal state, and use its predictions to guide some heuristic search algorithm during inference. \citet{mcaleer2018solving} and \citet{agostinelli2019solving} use reinforcement learning to learn the value function via deep approximate value iteration~\citep{sutton1998reinforcement}. Then, \citet{mcaleer2018solving} use the learned value within Monte Carlo Tree Search~\citep{coulom2006efficient}, and~\citet{agostinelli2019solving} in a variation of A* algorithm~\citep{ebendt2009weighted} to search for a path to the goal state. A recent work of \citet{chervov2025rubik} generates non-backtracking random walks~\citep{alon2007non} from the goal state, approximates the average walk distances with a neural network, and uses them to guide beam search. Closest to our approach, \cite{takano2023selfsupervision} learns a policy directly by predicting the last move of random scrambles, which can be interpreted as the backward policy of a fixed random walk forward process. Our method instead learns the forward policy jointly to minimize the flow. 

It is important to note that there exist classical solvers for some of the considered puzzles that can provide optimal shortest path solutions. 3x3x3 Rubik's Cube can be solved by IDA* \citep{korf1985depth} with an admissible heuristic computed from a pattern database (PDB). However, such human-engineered PDBs demand significant computation and sophisticated domain knowledge about the task to construct, and can take up to hundreds of gigabytes of memory~\citep{culberson1998pattern}. Moreover, such approaches currently do not scale to larger problems like 4x4x4 Rubik's Cube, where graph size differs in many orders of magnitude from the 3x3x3 case. Learning-based approaches, on the contrary, require no expert knowledge about the task, which is also the case for our approach. The goal is not to replace exact solvers in domains with highly optimized heuristics, but to provide a scalable alternative that can be applied in new tasks where heuristics are unavailable or costly to design. 

\section{Learning Shortest Paths with GFlowNets}\label{sec:methodology}

\subsection{Conditions for Minimizing Trajectory Length}
\label{sec:maintheory}

In this section, we theoretically examine non-acyclic GFlowNets that minimize their respective expected trajectory length. We prove that $\E[n_\tau]$ is minimized if and only if $\PF$ and $\PB$ move exclusively across shortest paths between $s_0$ and terminal states. We denote $\cl(s')$ to be the length of a shortest path from $s_0$ to $s'$; note that by Assumption~\ref{graph_assumption}, such a path always exists.

To define a probability distribution over $\cT$, \citet{morozov2025revisiting} assume that $\PB(s \mid s') > 0$ for any transition $(s \to s')$. As we will show, minimizing $\E[n_\tau]$ requires $\PB$ to assign zero probabilities to transitions that do not lie on a shortest path. Therefore we replace the assumption $\PB(s \mid s') > 0$ with a more general and natural assumption that requires $\E[n_\tau] < +\infty$. Here $\E[n_\tau]$ is the average number of steps a backward random walk starting in $s_f$ with transition probabilities given by $\PB$ takes to reach $s_0$.
\begin{assumption}
\label{pb_assumption}
    $\PB$ satisfies $\E[n_\tau] < +\infty$.
\end{assumption}
This assumption implies that $\PB$ can therefore be viewed as a transition kernel of an absorbing Markov chain~\citep{kemeny1969finite}, and defines a valid probability distribution over $\cT$ with probability function $\cP(\tau) = \prod_{t = 0}^{n_\tau} \PB(s_t \mid s_{t + 1})$. For a formal derivation and detailed discussion, we refer the reader to Appendix~\ref{app:assumpt}. Next, we provide two lemmas that link $\E[n_\tau]$ to the lengths of shortest paths.

\begin{lemma}
\label{lemma:lower_bound} 
Grant \Cref{graph_assumption}. Then for any backward policy $\PB$ that satisfies Assumption~\ref{pb_assumption}:
\begin{equation}
\label{eq:lower_bound}
\E[n_\tau] \ge \sum\limits_{x \in \cX} \PB(x \mid s_f) \cl(x)\,.
\end{equation}
\end{lemma}
\vspace{-0.15cm}
This is a rather intuitive result, stating that $\E[n_\tau]$ is lower bounded by the expected length of a shortest path from $s_0$ to a terminal state weighted by $\PB(x \mid s_f)$. We refer the reader to Appendix~\ref{app:lower_bound_proof} for the full proof. Next, we show that there always exists a backward policy that achieves equality in~\eqref{eq:lower_bound}. Recall that we write $\cR(x)$ for GFlowNet reward.

\begin{lemma}
\label{lemma:best_pb_existence}
Grant \Cref{graph_assumption}. Then, there always exists a backward policy that satisfies the reward matching condition $\PB(s \mid s_f) = \cR(s) / \cZ$ and induces the trajectory distribution with expected trajectory length equal to
    \begin{equation}\label{eq:lower_bound_exact}
    \E[n_\tau] = \sum\limits_{x \in \cX} \frac{\cR(x)}{\cZ} \cl(x)\,.
    \end{equation}
\end{lemma} \vspace{-0.25cm}
The proof is constructive, showing that a backward policy that always travels along shortest paths between $s_0$ and terminal states has the smallest possible $\E[n_\tau]$; we refer to Appendix~\ref{app:lemma_proofs} for a detailed derivation. Finally, combining the results of the two lemmas, we state the main theorem.

\begin{theorem}
\label{theorem:smalles_length}
Grant \Cref{graph_assumption} and let $\PB$ be a backward policy that satisfies the reward matching condition $\PB(s \mid s_f) = \cR(s) / \cZ$ and satisfies Assumption~\ref{pb_assumption}. Then $\PB$ minimizes $\E[n_\tau]$ if and only if, for each trajectory $\tau \in \cT$ that terminates at any state $x \in \cX$, $n_\tau \neq \cl(x)$ implies $\cP(\tau) = 0$. In other words, minimization of expected trajectory length is equivalent to assigning zero probability to all trajectories that are not shortest paths from $s_0$ to a terminal state.
\end{theorem}
We refer the reader to Appendix~\ref{app:proofs} for a detailed proof. Intuitively, we have shown that if we have a backward policy that minimizes $\E[n_\tau]$, sampling a trajectory from it means sampling a terminal state $x$ from the probability distribution $\cR(x') / \cZ$, and then sampling a shortest path between $s_0$ and $x$ in the backward direction. It is important to note that such $\PB$ is generally not unique: for any $s \in \cS \setminus \{ s_f \}$, it can assign arbitrary probabilities to parents of $s$ that lie on a shortest path between $s_0$ and $s$ (but must assign zero probability to parents that do not lie on one).

\begin{remark}
\label{remark_pf}
    Proposition 3.8 of \citet{morozov2025revisiting} states that for any backward policy $\PB$, there exists a unique forward policy $\PF$ that induces the same probability distribution over $\cT$ as $\PB$, and vice versa. Therefore, Theorem~\ref{theorem:smalles_length} can be equivalently restated for forward policies as well, since the condition of the theorem is related only to the trajectory distribution induced by the policy, not the policy itself. 
\end{remark}

\vspace{-0.12cm}
\subsection{Application to Pathfinding in Graphs}
\label{sec:construction}
\vspace{-0.12cm}

Building on the presented theory, we next show how GFlowNets can be applied to find shortest paths in general graphs. Let $G = (V, E)$ be an arbitrary finite directed graph. The task is to find shortest paths from every vertex $v$ to a specified goal vertex $v_g$. We note that the pathfinding problem on unweighted graphs is considered, and assume that a path from every vertex to the goal vertex exists.

The core idea of our approach is to slightly modify the graph $G$ to define such a non-acyclic GFlowNet environment, where $\PB$ with the smallest $\E[n_\tau]$ solves the pathfinding task of interest. First, we remove all edges outgoing from $v_g$ in $G$. This does not change the task since no such edge lies on a shortest path to $v_g$. Then, we define a non-acyclic GFlowNet environment $\cG = (\cS, \cE)$, where \vspace{-0.15cm}
\begin{enumerate}[itemsep=-2pt,leftmargin=8pt]
    \item States $\cS$ correspond to vertices in $G$, with an addition of a special sink state $s_f$. Initial state specifically $s_0$ corresponds to $v_g$;
    \item Transitions $\cE$ correspond to reversed edges from $G$. In addition, a transition from each state to $s_f$ is added (with an exception of $s_f$ itself).
\end{enumerate}
\vspace{-0.15cm}

This is a valid environment that satisfies Assumption~2.1: $s_0$ has no incoming edges and
$s_f$ no outgoing edges, every state has a transition to $s_f$ (so all states except $s_f$ are terminal), and every state is reachable from $s_0$ since every vertex of $G$ has a path to $v_g$.

Let $\cR$ be any positive reward function defined on $\cS \setminus \{ s_f \}$. Then, by Theorem~\ref{theorem:smalles_length}, a backward policy with the smallest $\E[n_\tau]$ that satisfies the reward matching condition with respect to $\cR$, will assign non-zero probabilities only to shortest paths between $s_0$ and terminal states, where the former coincides with $v_g$, and the latter coincide with all vertices in the graph $G$. Since $\PB$ traverses $\cG$ in reverse direction, and transitions in $\cG$ are reversed edges from $G$, starting in arbitrary state $s \in \cS \setminus \{ s_f \}$ and sampling a trajectory via $\PB$ will result in a shortest path from the vertex $v$ (corresponding to $s$) to $v_g$ (corresponding to $s_0$).

Therefore, training a non-acyclic GFlowNet that minimizes its expected trajectory length will result in a model that solves the pathfinding task in the graph $G$, specifically via the obtained backward policy $\PB$. Figure~\ref{fig:main_fig} visualizes the overall construction. 

\begin{figure}[t!]
    \centering
    \includegraphics[width=0.96\linewidth]{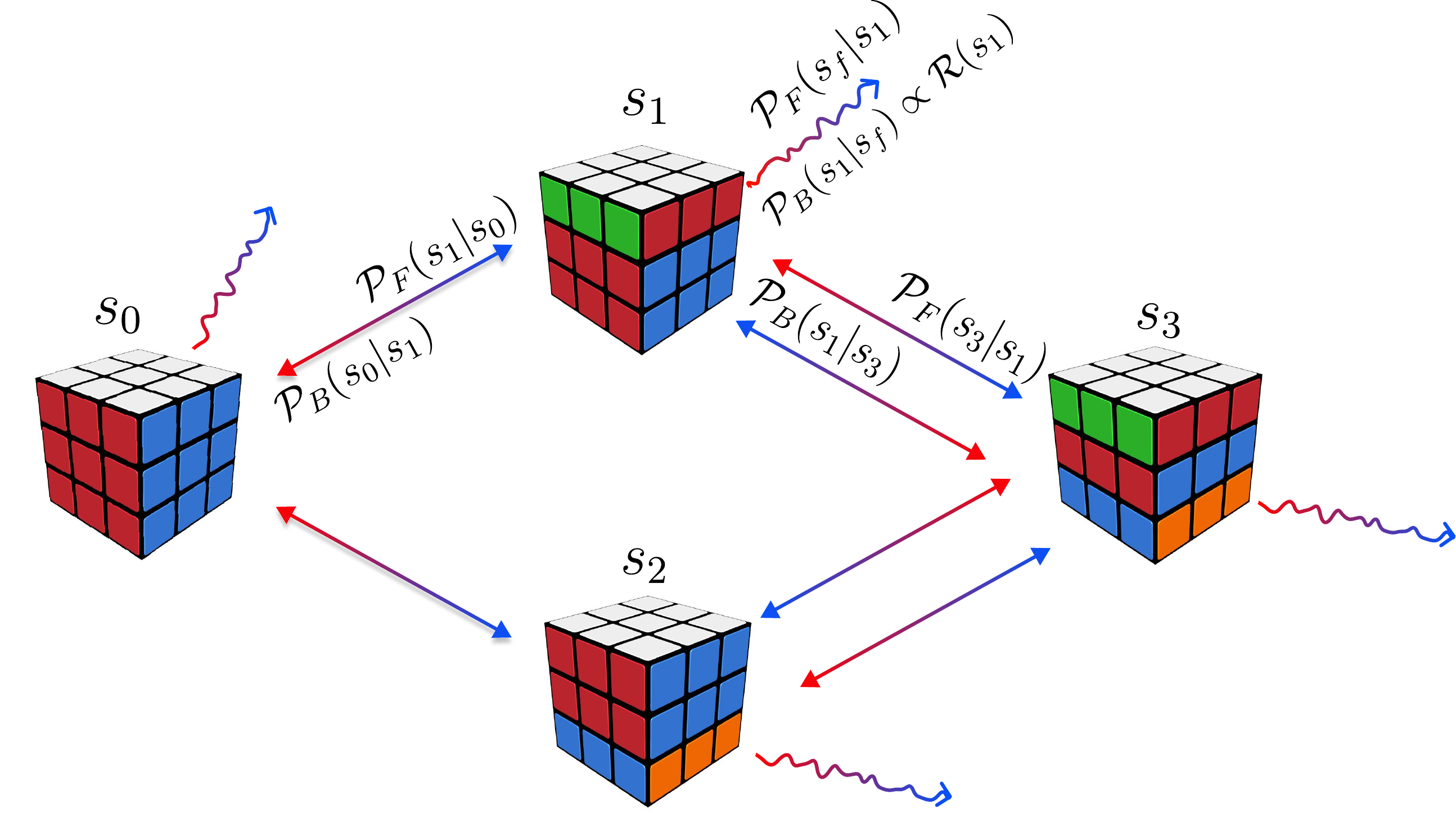}
    \caption{Visualization of the GFlowNet environment in the proposed construction using Rubik's Cube as an example. Forward transitions are denoted by blue arrows, backward ones are denoted by red arrows. Note that backward transitions from the goal state $s_0$ are removed in our construction — it is an absorbing state with respect to $\PB$, while there are forward and backward transitions in both directions between $s_1$ and $s_3$. Each state has a forward transition to $s_f$, corresponding to the stop action, while $\PB(s \mid s_f)$ corresponds to the reward distribution, which we set as uniform.}\label{fig:main_fig}
    \vspace{-0.3cm}
\end{figure}

\textbf{Choice of rewards.} While the presented construction allows for any positive reward function to be used, the most natural task-agnostic choice is a uniform reward distribution: $\cR(s) = 1$. Additionally, the optimal $\E[n_\tau]$ in this case coincides with the average length of a shortest path to the goal state $\frac{1}{|V|}\sum_{v \in V} \cl(v)$. We follow this design choice in the training algorithm discussed below and in all our experiments. 

\textbf{Forward policy.} The aim of the backward policy in the presented construction is to find shortest paths to the goal state $s_0$. The forward policy, however, is an auxiliary component needed for training and serves a different purpose. It starts in the goal state and traverses the graph in the other direction, aiming to sample states from the specified reward distribution, which is uniform in our case. It also has to assign a probability to stop in each state, corresponding to the transition to the sink state $\PF(s_f \mid s)$. For example, in the case of Rubik's Cubes, $\PB$ has to solve any possible configuration in the smallest possible number of moves, while $\PF$ has to scramble a solved cube in the smallest possible number of moves (including deciding when to stop scrambling), obtaining a uniformly distributed configuration. 

\textbf{Weighted graphs.} While our work focuses on pathfinding in unweighted graphs, the presented framework can be naturally theoretically extended to weighted graphs. If the edges have positive weights, one can show that the expected sum of edge weights along the trajectory coincides with the total normalized sum of weighted edge flows. Following the same theoretical construction we present for the unweighted case, one can show that the backward policy minimizing this quantity recovers shortest paths in the weighted graph, and the only algorithmic difference will be replacing the state flow regularizer with the weighted edge flow regularizer. We present an extended discussion and a small experimental demonstration in Appendix~\ref{app:weighted}. We consider a deeper examination and evaluation on large-scale weighted graphs an important direction for future research.

\vspace{-0.15cm}
\subsection{Training Algorithm}\label{sec:alg}

\vspace{-0.15cm}


\citet{morozov2025revisiting} trained non-acyclic GFlowNets on-policy with detailed balance loss~\citep{bengio2023gflownet} and state flow regularization. We make two changes. First, since $\mathbb{E}[n_\tau]$ can be very large early in training, we bound the length of sampled trajectories, as in \citet{brunswic2024theory}. Second,
training with detailed balance converged very slowly in our experiments, so we use a modification of trajectory balance loss~\citep{malkin2022trajectory}, which has more
efficient credit assignment. In our setting, the main training signal comes from the goal state, which every trajectory contains but most individual
transitions do not.

\vspace{-0.1cm}
\subsubsection{Flow-Regularized Trajectory Balance}
\vspace{-0.1cm}

In detail, our approach is as follows. The model parameterizes and learns forward and backward policies: $\PF(s' \mid s, \theta)$ and $\PB(s \mid s', \theta)$. At each step, we generate a batch of partial trajectories of a fixed length $N_{\text{max}}$ from the forward policy $\PF$ starting in $s_0$. To ensure that the trajectory does not terminate earlier, we simply mask the stop action $\PF(s_f \mid s, \theta)$ during sampling. Denote such a partial trajectory as $\tau' = (s_0 \to s_1 \to ... \to s_{N_{\text{max}}})$. Since each state in the environment is terminal (see Section~\ref{sec:construction}), each prefix $\tau'_{0:i}$ with an addition of a terminal transition $(s_i \to s_f)$ can be viewed as a complete trajectory $(s_0 \to s_1 \to ... \to s_i \to s_f) \in \cT$. We then compute trajectory balance loss \citep{malkin2022trajectory} for all prefixes and take the sum. The motivation for this design choice is that considering the balance conditions on all prefixes instead of only the whole trajectory produces more information about the sampled trajectory, resulting in a more sample-efficient approach.

To add state flow regularization~\citep{brunswic2024theory, morozov2025revisiting} to the loss, one can note that in environments containing a transition to $s_f$ from every other state, which is the case in our construction, the flow can be expressed as $\cF(s) = \cR(s) / \PF(s_f \mid s)$, and does not need to be learned as an additional component, see, e.g., \citet{deleu2022bayesian}. Theoretically, this parameterization can be obtained by combining the conditions on terminal edge flows $\cF(s \to s_f) = \cR(s)$ and a part of detailed balance conditions $\cF(s \to s_f) = \cF(s) \PF(s_f \mid s)$ (Proposition 3.10 and Proposition 3.8 of \citet{morozov2025revisiting}).

Combining the choice of the loss, uniform rewards and the parameterization of the flows discussed above, we obtain the following regularized training objective:
\begin{align} \label{eq:reg_tb_loss}
    \mathcal{L}_{\mathtt{regTB}}(\theta, \tau) = \sum\limits_{i = 0}^{N_{\text{max}}}\left\{ \cL_{\mathtt{TB}}(\theta, \tau_{0:i}) +  \frac{\lambda }{\PF(s_f \mid s_i, \theta)}\right\}\,,
\end{align} \vspace{-0.2cm}
where 
\begin{align*}
    \cL_{\mathtt{TB}}(\theta, \tau_{0:i}) = \bigg( \log \frac{\PF(s_f | s_i, \theta) \prod_{t=1}^i  \PF(s_t | s_{t - 1}, \theta)}{(1/|V|)  \prod_{t=1}^i  \PB(s_{t - 1} | s_{t}, \theta)} \bigg)^2\,.
\end{align*}
In general, the normalizing constant is unknown in GFlowNets, and a learnable scalar is plugged into the loss. However, in the case of our reward choice $\cR(s) = 1$, it simply corresponds to the number of vertices in the graph $\cZ = |V|$, thus does not need to be learned. In addition, as proposed in~\cite{morozov2025revisiting}, we do not apply the regularization to the source state flow.

The full training procedure is summarized in Algorithm~\ref{alg:train}. An important note here is that optimization with respect to both $\PF$ and $\PB$ is a crucial part of our algorithm. Specialized methods for optimizing backward policies exist in GFlowNet literature~\citep{jang2024pessimistic, gritsaev2024optimizing}, and can be potentially applied to further improve training. 

\subsubsection{Relation Between the Training Objective and the Theory}

Theorem~\ref{theorem:smalles_length} concerns a constrained problem: minimizing total flow over GFlowNets that satisfy balance and reward matching. Like existing training
procedures for non-acyclic GFlowNets \citep{brunswic2024theory, morozov2025revisiting}, Algorithm~\ref{alg:train} solves it only approximately. The constraint is replaced by a penalty, as in Equations~(11) and (12) of \citet{morozov2025revisiting}, so for fixed $\lambda > 0$ the minimizer of~\eqref{eq:reg_tb_loss} need not satisfy balance exactly. The regularizer is evaluated on
states visited by the stop-masked forward policy, so in expectation it penalizes a visitation-weighted flow rather than the total flow (see Appendix B.1 of~\cite{morozov2025revisiting} for discussion), and balance is enforced only within $N_{\max}$ steps of $s_0$. Theoretically quantifying the resulting bias remains an open problem for non-acyclic GFlowNet training in general. However, only $\mathcal{P}_B$ is used for pathfinding, and it satisfies reward matching by construction, since $\mathcal{P}_B(\cdot \mid s_f)$ is uniform. The quantity
minimized in Theorem~\ref{theorem:smalles_length} is thus the expected length of a backward walk from a uniformly random state, which we measure directly in Section~\ref{sec:exp_perms} (curve ``Policy'' in Figure~2): after training, it is within 0.6\% and 2.4\% of the optimum for $n = 15$ and $n = 20$.

\begin{algorithm}[t!]
   \caption{Training a GFlowNet for pathfinding with regularized trajectory balance objective}
   \label{alg:train}
\begin{algorithmic}[1]
   \STATE {\bfseries Input:} Model and optimizer, regularization coefficient $\lambda$, maximum length of training trajectories $N_{\text{max}}$, batch size $B$
   \REPEAT
   \STATE Sample a batch of trajectories $\tau^{1:B}$ of fixed length $N_{\text{max}}$ from $\PF(s' \mid s, \theta)$ starting in $s_0$
   \STATE Compute $\frac{1}{B}\sum_{i = 1}^B \nabla_\theta \mathcal{L}_{\mathtt{regTB}}(\theta, \tau^i)$  and update model parameters $\theta$
   \UNTIL{convergence}
\end{algorithmic}
\label{gfn_ments}
\end{algorithm}

\vspace{-0.15cm}
\subsection{Beam Search}
\vspace{-0.15cm}

While in theory, the optimal backward policy with the smallest total flow finds shortest paths exactly, the learned policy will only provide approximate solutions in case of larger graphs in practice, as further demonstrated in our experimental evaluation in Section~\ref{sec:experiments}. Following~\citet{chervov2025rubik}, we use beam search as a heuristic to improve the solution at test-time. 

Let $W$ be the width of beam search. At each step of the search, we consider all ways to continue each of the currently stored trajectories, and keep top $W$ candidates based on the logarithm of the product of backward transition probabilities along the trajectory. The search stops when the goal state $s_0$ is reached. The resulting algorithm is very similar to the one widely applied in text generation models~\citep{sutskever2014sequence, meister2020if}, and can be viewed as an approximate way to find a trajectory with the highest probability assigned by the model. A corner case of the algorithm when beam width $W$ is set to 1 corresponds to the greedy evaluation of the learned backward policy: instead of faithfully sampling a trajectory from $\PB$, at each step one chooses a transition with the highest probability $\operatorname{argmax}_s \PB(s \mid s', \theta)$. If the learned policy is optimal, this will still produce shortest paths, as zero probability must be assigned to all transitions that do not lie on a shortest path (Theorem~\ref{theorem:smalles_length}). 

\vspace{-0.2cm}
\section{Experiments}\label{sec:experiments}
\vspace{-0.2cm}

We conduct an experimental evaluation on a synthetic Swap task, as well as on 2x2x2, 3x3x3 and 4x4x4 Rubik's Cubes. Our neural network architecture is very similar to the one employed in~\citet{chervov2025rubik}: MLP with residual connections and layer normalization. In each task, states correspond to permutations of a fixed size, and a one-hot encoding is provided as input to the neural network. While our theoretical construction is general and imposes no restrictions on the structure of the graph, the choice of Cayley graph problems for our experiments is motivated by the fact that they are standard challenging benchmarks in the context of ML-based combinatorial search and pathfinding~\citep{chervov2025rubik}. For comprehensive experimental details we refer to Appendix~\ref{app:exp_details}.

\begin{figure}[!t]

    \centering
    \includegraphics[width=0.9\linewidth]{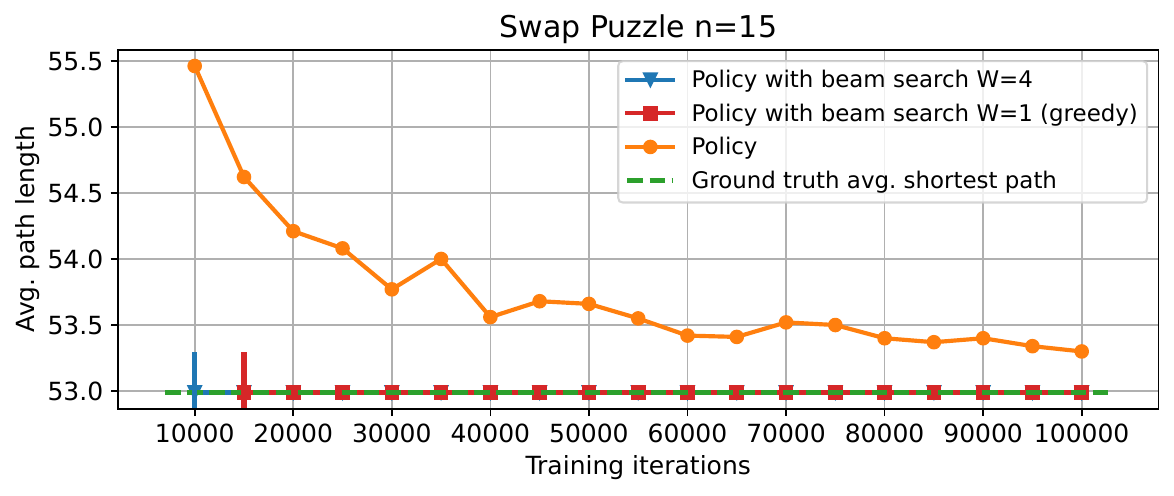}
    \includegraphics[width=0.9\linewidth]{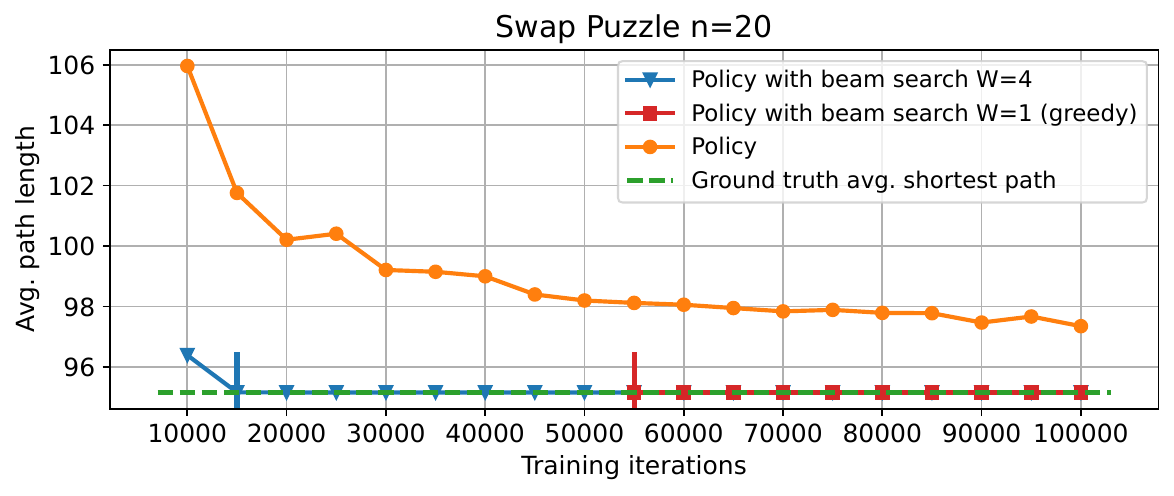}

    \caption{The results for Swap Puzzle $n=15$ and $n=20$ for three evaluation protocols: 1) the path is faithfully sampled from $\PB$, 2) the path is obtained by a greedy evaluation of $\PB$, 3) the path is generated with beam search of a small width $W=4$. The results for greedy evaluation are not presented for checkpoints where it failed to find a valid solution for every example in the test set. For $n=20$, 100k training iterations took 15 minutes to run on an NVIDIA H200 GPU.} 
    \label{fig:permsort_main}
    \vspace{-0.3cm}
    
\end{figure}

\subsection{Swap Puzzle}\label{sec:exp_perms}

To evaluate the capability of the proposed approach to find shortest paths, as well as the ability to generalize to unseen states in large graphs, we first consider a synthetic environment, which is a variation of the Swap Puzzle described in~\citet{mulholland2016permutation}.

Given an arbitrary permutation of $n$ elements, the task is to sort it using the smallest possible number of moves, where at each step you are allowed to swap any pair of adjacent elements. This corresponds to pathfinding in the Cayley graph of the Symmetric group $\mathrm{S}_n$ with the generating set consisting of transpositions of adjacent elements. The solution is rather straightforward: at each step, one must swap any pair of adjacent elements that are in descending order. This strategy is optimal due to two facts: 1) the identity permutation (goal state) is the only permutation that has $0$ inversions, and 2) swapping a pair of adjacent elements can either decrease the number of inversions by one or increase it by one. Thus, the length of the solution for any permutation $s$ corresponds to the number of inversions in it. The goal is to verify that our model successfully learns the optimal strategy.

\begin{figure}[!t]

    \centering
    \includegraphics[width=0.9\linewidth]{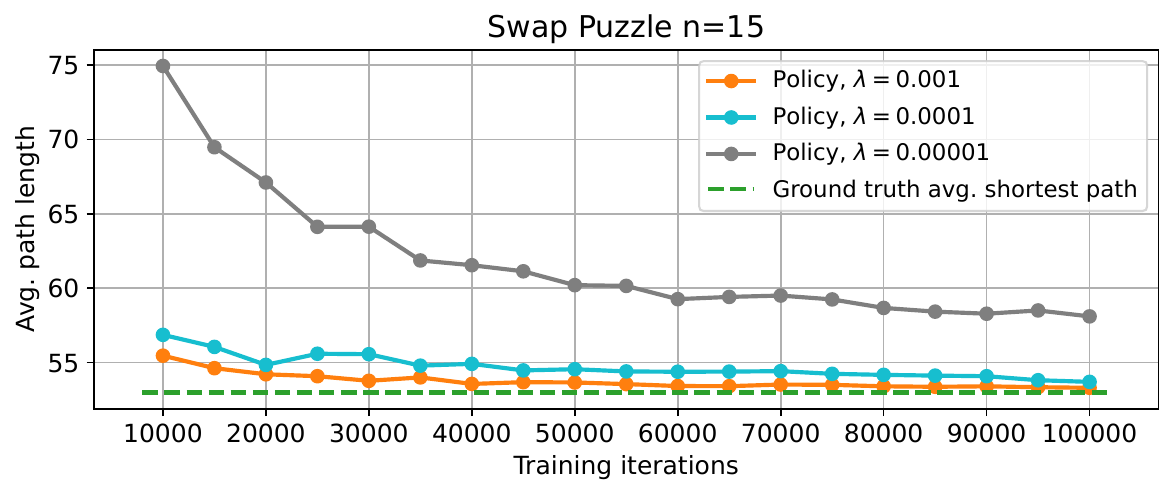}
    \includegraphics[width=0.9\linewidth]{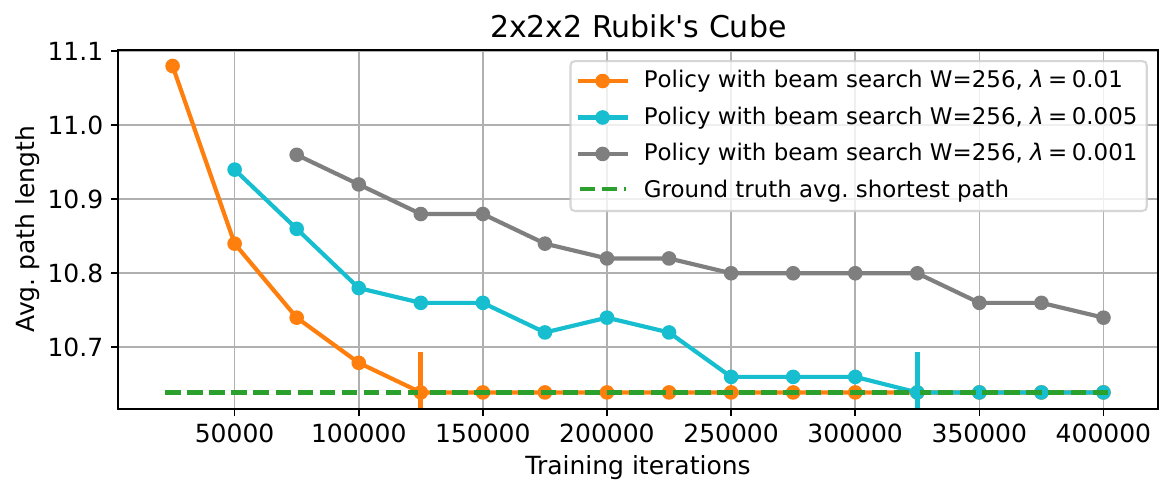}

    \caption{Effect of the regularization coefficient $\lambda$ on the obtained solution lengths on Swap Puzzle $n=15$ and 2x2x2 Rubik's cube. In both cases, choosing a larger $\lambda$ led to models that fail to find any valid paths.} 
    \label{fig:reg_abl}
    \vspace{-0.3cm}
    
\end{figure}

{\renewcommand{\arraystretch}{1.0}
\setlength{\tabcolsep}{3pt}
\begin{table*}[t]
  \centering
  \footnotesize
  \setlength{\tabcolsep}{3.5pt}
  \renewcommand{\arraystretch}{1.15}
  \begin{tabular}{@{}lcccc@{\hspace{8pt}}lcccc@{\hspace{8pt}}lcccc@{}}
    \toprule
    & \multicolumn{4}{c}{\textbf{2x2x2 Rubik's Cube}} &&
      \multicolumn{4}{c}{\textbf{3x3x3 Rubik's Cube}} &&
      \multicolumn{4}{c}{\textbf{4x4x4 Rubik's Cube}} \\
    \cmidrule(lr){2-5}\cmidrule(lr){7-10}\cmidrule(lr){12-15}
    & \multicolumn{2}{c}{\textbf{Ours}} & \multicolumn{2}{c}{\textbf{CayleyPy}} &&
      \multicolumn{2}{c}{\textbf{Ours}} & \multicolumn{2}{c}{\textbf{CayleyPy}} &&
      \multicolumn{2}{c}{\textbf{Ours}} & \multicolumn{2}{c}{\textbf{CayleyPy}} \\
    \cmidrule(lr){2-3}\cmidrule(lr){4-5}\cmidrule(lr){7-8}\cmidrule(lr){9-10}\cmidrule(lr){12-13}\cmidrule(lr){14-15}
    \makecell[l]{Beam\\size $W$}
      & \makecell[c]{Avg.\\length} & \makecell[c]{Solve\\rate}
      & \makecell[c]{Avg.\\length} & \makecell[c]{Solve\\rate}
    & \makecell[l]{Beam\\size $W$}
      & \makecell[c]{Avg.\\length} & \makecell[c]{Solve\\rate}
      & \makecell[c]{Avg.\\length} & \makecell[c]{Solve\\rate}
    & \makecell[l]{Beam\\size $W$}
      & \makecell[c]{Avg.\\length} & \makecell[c]{Solve\\rate}
      & \makecell[c]{Avg.\\length} & \makecell[c]{Solve\\rate} \\
    \midrule
    $2^{0}$  & \textbf{11.62} & \textbf{1.00} & x & 0.00 & $2^{0}$  & x & \textbf{0.471} & x & 0.000 & $2^{3}$  & x & \textbf{0.07} & x & 0.00 \\
    $2^{2}$  & \textbf{11.24} & \textbf{1.00} & x & 0.00 & $2^{3}$  & x & \textbf{0.984} & x & 0.001 & $2^{6}$  & x & \textbf{0.27} & x & 0.00 \\
    $2^{4}$  & \textbf{10.78} & \textbf{1.00} & x & 0.00 & $2^{6}$  & \textbf{25.33} & \textbf{1.000} & x & 0.687 & $2^{9}$  & x & \textbf{0.74} & x & 0.00 \\
    $2^{6}$  & \textbf{10.64} & \textbf{1.00} & x & 0.06 & $2^{9}$  & \textbf{23.49} & \textbf{1.000} & 24.34 & 1.000 & $2^{12}$  & \textbf{64.09} & \textbf{1.00} & x & 0.16 \\
    $2^{8}$  & \textbf{10.64} & \textbf{1.00} & x & 0.89 & $2^{12}$ & \textbf{22.42} & \textbf{1.000} & 22.44 & 1.000 & $2^{15}$ & \textbf{59.95} & \textbf{1.00} & x & 0.69 \\
    $2^{10}$ & \textbf{10.64} & \textbf{1.00} & \textbf{10.64} & \textbf{1.00} & $2^{15}$ & \textbf{21.61} & \textbf{1.000} & \textbf{21.61} & \textbf{1.000} & $2^{18}$ & \textbf{57.25} & \textbf{1.00} & 59.39 & {1.00} \\
    $2^{12}$ & \textbf{10.64} & \textbf{1.00} & \textbf{10.64} & \textbf{1.00} & $2^{18}$ & 21.20 & 1.000 & \textbf{21.15} & \textbf{1.000} & $2^{21}$ & 53.88 & {1.00} & \textbf{53.53} & \textbf{1.00} \\
    \bottomrule
  \end{tabular}
  \caption{Comparison with CayleyPy Cube~\citep{chervov2025rubik}. We present average lengths of solutions found by both methods for different beam search budgets. Solve rate corresponds to the fraction of examples in the test set for which the model has found a valid path to the goal within 100 steps. For the 2x2x2 Rubik's Cube, the average length of 10.64 is optimal for the test set, validated by the BFS algorithm. We focus on comparison to CayleyPy Cube as it was shown to thoroughly outperform previous methods in both solution lengths and runtime efficiency. For reference, the primary established RL-based approach DeepCubeA obtains the avg. length of 21.50 in 3x3x3 cube with the search budget of $W=2^{18}$.}
  \label{cube_table}
  \vspace{-0.1cm}
\end{table*}}

Figure~\ref{fig:permsort_main} presents the evaluation results for $n=15$ and $n=20$, corresponding to Cayley graphs with $\approx 1.3 \cdot 10^{12}$ and $\approx 2.4 \cdot 10^{18}$ states respectively. For each task, we fix a test set of $500$ uniformly sampled permutations, and report the average length of the path found by the model over the course of training. We plot the results for three evaluation protocols 1) the path is faithfully sampled from $\PB$, 2) the path is obtained by a greedy evaluation of $\PB$, 3) the path is generated with beam search of a small width $W=4$. We observe that after sufficient training, both greedy and beam search evaluation protocols produce exact shortest paths for every permutation in the test set (requiring less training to produce optimal solutions with beam search), and the policy itself produces solutions that are close to optimal on average. 

It is important to note that, e.g., in the case of $n=20$, the model in Figure~\ref{fig:permsort_main} has seen only $10^9$ out of overall $2.4 \cdot 10^{18}$ states during training, which shows its generalization capabilities. Moreover, we set the length $N_{\max} = 50$ for training trajectories, while uniformly random permutations of $n = 20$ elements have about 95 inversions on average, so almost all test states lie outside the region where balance is enforced during training. The same holds for the 4x4x4 cube (Section~\ref{sec:exp_rubik}), where $N_{\max} = 30$ and
the solutions found average more than 50 moves. Performance on these states is
due to generalization of the learned policy.

\subsection{Rubik's Cubes}\label{sec:exp_rubik}

Next, we evaluate our approach on 2x2x2, 3x3x3 and 4x4x4 Rubik's Cubes against CayleyPy Cube~\citep{chervov2025rubik}. It is the state-of-the-art machine learning approach designed specifically for solving Rubik's Cubes and similar puzzles, which was shown to outperform previous methods like DeepCubeA~\citep{agostinelli2019solving} and EfficientCube~\citep{takano2023selfsupervision} in both solution lengths and runtime efficiency. To ensure a fair comparison, for CayleyPy Cube, we train a model with neural network hyperparameters specified in the original work~\citep{chervov2025rubik} and a model with the same neural network size as ours, picking the best of the two. 

Both methods utilize beam search to look for the shortest paths once the model is trained, and we provide the comparison across different beam search widths $W$ in Table~\ref{cube_table}. While 2x2x2 cubes induce a relatively small graph with $\approx 88 \cdot 10^6$ states, 3x3x3 cubes present a challenging pathfinding task with $\approx 43 \cdot 10^{18}$ states, and 4x4x4 cubes result in an extremely large graph with $\approx 17 \cdot 10^{46}$ states. For the 2x2x2 Rubik's Cube, our approach requires a beam search width 16 times smaller to find optimal solutions, as well as provides valid solutions for the entire test set even with greedy evaluation ($W=1$), whereas CayleyPy cube fails to find any valid paths for smaller beam search widths. For the 3x3x3 Rubik's Cube, our model also shows better performance across smaller beam search widths from $1$ to $2^9$, while showing comparable results for larger values of $W=\{2^{12},\; 2^{15},\; 2^{18}\}$. For the hardest 4x4x4 case, our model shows better performance across all beam widths from $2^3$ to $2^{18}$, while slightly falling behind for $W=2^{21}$. It is worth noting that our model recovers valid paths for the whole test set with beam width as small as $2^{12} = 4096$, while CayleyPy Cube requires at least $2^{18} = 262144$ beams.

\textbf{Runtime efficiency.} In addition, we evaluate the runtime efficiency of both models. When run on a single NVIDIA H200 GPU, for 3x3x3 cubes with $W=2^{18}$ our trained model takes 1.74 seconds on average to solve a single configuration with a 25M-parameter neural network, while CayleyPy Cube takes 6.19 seconds on average with a 4M-parameter neural network. For 4x4x4 cubes with $W=2^{21}$, our model takes 50.7 seconds on average vs 330.6 seconds for CayleyPy Cube. The main reason for speedup even when a larger neural network is used is the fact that CayleyPy Cube and similar approaches like DeepCubeA~\citep{agostinelli2019solving} require $N$ times more neural network evaluations with the same beam width, where $N$ is the number of neighbors of a state (12 for 3x3x3 and 24 for 4x4x4). Indeed, mentioned approaches have to run a forward pass for every neighbor of a state to estimate values or distances, while our model outputs backward policy logits corresponding to every neighbor with a single forward pass.

\subsection{Ablation of Regularization Coefficient}\label{sec:exp_reg}


Finally, we discuss the role of the regularization coefficient $\lambda$ (see~\eqref{eq:reg_tb_loss}), which is one of the main hyperparameters in non-acyclic GFlowNet training with flow regularization. In our experiments, we observed that models perform better when trained with a larger value of $\lambda$, but can fail when $\lambda$ is too large. The effects of $\lambda$ are shown in Figure~\ref{fig:reg_abl}. Because of this, a simple rule of thumb can be used to quickly tune $\lambda$: evaluate the model after a small number of training iterations, and choose the largest $\lambda$ that produces a model able to find valid paths to the goal state.

\section{Conclusion}\label{sec:conclusion}

In this work, we establish a direct theoretical connection between non-acyclic GFlowNets and pathfinding. Our main result shows that minimizing expected trajectory length is equivalent to concentrating probability mass exclusively on shortest paths. This provides a new interpretation of flow minimization in non-acyclic GFlowNets and reframes shortest-path optimality in probabilistic terms. Based on this result, we show how non-acyclic GFlowNets with flow regularization can be used to train a stochastic policy to find shortest paths in unweighted graphs, and demonstrate strong capabilities of this approach in our experimental evaluation.

\subsubsection*{Acknowledgements}
This research was supported in part through computational resources of HPC facilities at HSE University~\citep{kostenetskiy2021hpc}.

\bibliography{bibliography}

\begin{thebibliography}{}

\bibitem[Agostinelli et~al., 2019]{agostinelli2019solving}
Agostinelli, F., McAleer, S., Shmakov, A., and Baldi, P. (2019).
\newblock Solving the rubik’s cube with deep reinforcement learning and search.
\newblock {\em Nature Machine Intelligence}, 1(8):356--363.

\bibitem[Alon et~al., 2007]{alon2007non}
Alon, N., Benjamini, I., Lubetzky, E., and Sodin, S. (2007).
\newblock Non-backtracking random walks mix faster.
\newblock {\em Communications in Contemporary Mathematics}, 9(04):585--603.

\bibitem[Atanackovic et~al., 2024]{atanackovic2024dyngfn}
Atanackovic, L., Tong, A., Wang, B., Lee, L.~J., Bengio, Y., and Hartford, J.~S. (2024).
\newblock Dyngfn: Towards bayesian inference of gene regulatory networks with gflownets.
\newblock {\em Advances in Neural Information Processing Systems}, 36.

\bibitem[Ba et~al., 2016]{ba2016layer}
Ba, J.~L., Kiros, J.~R., and Hinton, G.~E. (2016).
\newblock Layer normalization.
\newblock {\em arXiv preprint arXiv:1607.06450}.

\bibitem[Bengio et~al., 2021]{bengio2021flow}
Bengio, E., Jain, M., Korablyov, M., Precup, D., and Bengio, Y. (2021).
\newblock Flow network based generative models for non-iterative diverse candidate generation.
\newblock {\em Advances in Neural Information Processing Systems}, 34:27381--27394.

\bibitem[Bengio et~al., 2023]{bengio2023gflownet}
Bengio, Y., Lahlou, S., Deleu, T., Hu, E.~J., Tiwari, M., and Bengio, E. (2023).
\newblock Gflownet foundations.
\newblock {\em Journal of Machine Learning Research}, 24(210):1--55.

\bibitem[Bradbury et~al., 2018]{jax2018github}
Bradbury, J., Frostig, R., Hawkins, P., Johnson, M.~J., Leary, C., Maclaurin, D., Necula, G., Paszke, A., Vander{P}las, J., Wanderman-{M}ilne, S., and Zhang, Q. (2018).
\newblock {JAX}: composable transformations of {P}ython+{N}um{P}y programs.

\bibitem[Brunswic et~al., 2024]{brunswic2024theory}
Brunswic, L., Li, Y., Xu, Y., Feng, Y., Jui, S., and Ma, L. (2024).
\newblock A theory of non-acyclic generative flow networks.
\newblock In {\em Proceedings of the AAAI Conference on Artificial Intelligence}, volume~38, pages 11124--11131.

\bibitem[Chervov et~al., 2026]{chervov2025rubik}
Chervov, A., Khoruzhii, K., Bukhal, N., Naghiyev, J., Zamkovoy, V., Koltsov, I., Cheldieva, L., Sychev, A., Lenin, A., Obozov, M., et~al. (2026).
\newblock A machine learning approach that beats rubik's cubes.
\newblock {\em Advances in Neural Information Processing Systems}, 38:172509--172531.

\bibitem[Chervov et~al., 2025]{chervov2025cayleypy}
Chervov, A., Soibelman, A., Lytkin, S., Kiselev, I., Fironov, S., Lukyanenko, A., Dolgorukova, A., Ogurtsov, A., Petrov, F., Krymskii, S., et~al. (2025).
\newblock Cayleypy rl: Pathfinding and reinforcement learning on cayley graphs.
\newblock {\em arXiv preprint arXiv:2502.18663}.

\bibitem[Cooperman et~al., 1990]{cooperman1990applications}
Cooperman, G., Finkelstein, L., and Sarawagi, N. (1990).
\newblock Applications of cayley graphs.
\newblock In {\em International symposium on applied algebra, algebraic algorithms, and error-correcting codes}, pages 367--378. Springer.

\bibitem[Coulom, 2006]{coulom2006efficient}
Coulom, R. (2006).
\newblock Efficient selectivity and backup operators in monte-carlo tree search.
\newblock In {\em International conference on computers and games}, pages 72--83. Springer.

\bibitem[Cretu et~al., 2025]{cretu2025synflownet}
Cretu, M., Harris, C., Igashov, I., Schneuing, A., Segler, M., Correia, B., Roy, J., Bengio, E., and Lio, P. (2025).
\newblock {SynFlowNet}: Design of diverse and novel molecules with synthesis constraints.
\newblock In {\em The Thirteenth International Conference on Learning Representations}.

\bibitem[Culberson and Schaeffer, 1998]{culberson1998pattern}
Culberson, J.~C. and Schaeffer, J. (1998).
\newblock Pattern databases.
\newblock {\em Computational Intelligence}, 14(3):318--334.

\bibitem[Deleu et~al., 2022]{deleu2022bayesian}
Deleu, T., G{\'o}is, A., Emezue, C., Rankawat, M., Lacoste-Julien, S., Bauer, S., and Bengio, Y. (2022).
\newblock Bayesian structure learning with generative flow networks.
\newblock In {\em Uncertainty in Artificial Intelligence}, pages 518--528. PMLR.

\bibitem[Deleu et~al., 2023]{deleu2023joint}
Deleu, T., Nishikawa-Toomey, M., Subramanian, J., Malkin, N., Charlin, L., and Bengio, Y. (2023).
\newblock Joint bayesian inference of graphical structure and parameters with a single generative flow network.
\newblock {\em Advances in neural information processing systems}, 36:31204--31231.

\bibitem[Deleu et~al., 2024]{deleu2024discrete}
Deleu, T., Nouri, P., Malkin, N., Precup, D., and Bengio, Y. (2024).
\newblock Discrete probabilistic inference as control in multi-path environments.
\newblock In {\em The 40th Conference on Uncertainty in Artificial Intelligence}.

\bibitem[Ebendt and Drechsler, 2009]{ebendt2009weighted}
Ebendt, R. and Drechsler, R. (2009).
\newblock Weighted {A*} search--unifying view and application.
\newblock {\em Artificial Intelligence}, 173(14):1310--1342.

\bibitem[Gallo and Pallottino, 1988]{gallo1988shortest}
Gallo, G. and Pallottino, S. (1988).
\newblock Shortest path algorithms.
\newblock {\em Annals of operations research}, 13(1):1--79.

\bibitem[Gritsaev et~al., 2025]{gritsaev2024optimizing}
Gritsaev, T., Morozov, N., Samsonov, S., and Tiapkin, D. (2025).
\newblock Optimizing backward policies in {GF}lownets via trajectory likelihood maximization.
\newblock In {\em The Thirteenth International Conference on Learning Representations}.

\bibitem[Hart et~al., 1968]{hart1968formal}
Hart, P.~E., Nilsson, N.~J., and Raphael, B. (1968).
\newblock A formal basis for the heuristic determination of minimum cost paths.
\newblock {\em IEEE transactions on Systems Science and Cybernetics}, 4(2):100--107.

\bibitem[Hu et~al., 2023]{hu2023amortizing}
Hu, E.~J., Jain, M., Elmoznino, E., Kaddar, Y., Lajoie, G., Bengio, Y., and Malkin, N. (2023).
\newblock Amortizing intractable inference in large language models.
\newblock In {\em The Twelfth International Conference on Learning Representations}.

\bibitem[Ioffe and Szegedy, 2015]{ioffe2015batch}
Ioffe, S. and Szegedy, C. (2015).
\newblock Batch normalization: Accelerating deep network training by reducing internal covariate shift.
\newblock In {\em International conference on machine learning}, pages 448--456. pmlr.

\bibitem[Jang et~al., 2024]{jang2024pessimistic}
Jang, H., Jang, Y., Kim, M., Park, J., and Ahn, S. (2024).
\newblock Pessimistic backward policy for {GF}low{N}ets.
\newblock In {\em Advances in Neural Information Processing Systems}, volume~37, pages 107087--107111.

\bibitem[Karur et~al., 2021]{karur2021survey}
Karur, K., Sharma, N., Dharmatti, C., and Siegel, J.~E. (2021).
\newblock A survey of path planning algorithms for mobile robots.
\newblock {\em Vehicles}, 3(3):448--468.

\bibitem[Kemeny and Snell, 1969]{kemeny1969finite}
Kemeny, J.~G. and Snell, J.~L. (1969).
\newblock {\em Finite markov chains}, volume~26.
\newblock van Nostrand Princeton, NJ.

\bibitem[Kim et~al., 2025]{kim2025ant}
Kim, M., Choi, S., Kim, H., Son, J., Park, J., and Bengio, Y. (2025).
\newblock Ant colony sampling with gflownets for combinatorial optimization.
\newblock In {\em International Conference on Artificial Intelligence and Statistics}, pages 469--477. PMLR.

\bibitem[Korf, 1985]{korf1985depth}
Korf, R.~E. (1985).
\newblock Depth-first iterative-deepening: An optimal admissible tree search.
\newblock {\em Artificial intelligence}, 27(1):97--109.

\bibitem[Kostenetskiy et~al., 2021]{kostenetskiy2021hpc}
Kostenetskiy, P., Chulkevich, R., and Kozyrev, V. (2021).
\newblock Hpc resources of the higher school of economics.
\newblock In {\em Journal of Physics: Conference Series}, volume 1740, page 012050. IOP Publishing.

\bibitem[Koziarski et~al., 2024]{koziarski2024rgfn}
Koziarski, M., Rekesh, A., Shevchuk, D., van~der Sloot, A., Gai{\'n}ski, P., Bengio, Y., Liu, C., Tyers, M., and Batey, R. (2024).
\newblock Rgfn: Synthesizable molecular generation using gflownets.
\newblock {\em Advances in Neural Information Processing Systems}, 37:46908--46955.

\bibitem[Loshchilov and Hutter, 2019]{loshchilov2017decoupled}
Loshchilov, I. and Hutter, F. (2019).
\newblock Decoupled weight decay regularization.
\newblock In {\em International Conference on Learning Representations (ICLR)}.

\bibitem[Madan et~al., 2023]{madan2023learning}
Madan, K., Rector-Brooks, J., Korablyov, M., Bengio, E., Jain, M., Nica, A.~C., Bosc, T., Bengio, Y., and Malkin, N. (2023).
\newblock Learning gflownets from partial episodes for improved convergence and stability.
\newblock In {\em International Conference on Machine Learning}, pages 23467--23483. PMLR.

\bibitem[Madkour et~al., 2017]{madkour2017survey}
Madkour, A., Aref, W.~G., Rehman, F.~U., Rahman, M.~A., and Basalamah, S. (2017).
\newblock A survey of shortest-path algorithms.
\newblock {\em arXiv preprint arXiv:1705.02044}.

\bibitem[Malkin et~al., 2022]{malkin2022trajectory}
Malkin, N., Jain, M., Bengio, E., Sun, C., and Bengio, Y. (2022).
\newblock Trajectory balance: Improved credit assignment in gflownets.
\newblock {\em Advances in Neural Information Processing Systems}, 35:5955--5967.

\bibitem[Malkin et~al., 2023]{malkin2022gflownets}
Malkin, N., Lahlou, S., Deleu, T., Ji, X., Hu, E.~J., Everett, K.~E., Zhang, D., and Bengio, Y. (2023).
\newblock {GF}lownets and variational inference.
\newblock In {\em The Eleventh International Conference on Learning Representations}.

\bibitem[Manta et~al., 2023]{manta2023gflownets}
Manta, D.~C., Hu, E.~J., and Bengio, Y. (2023).
\newblock Gflownets for causal discovery: an overview.
\newblock In {\em ICML 2023 Workshop on Structured Probabilistic Inference $\&$ Generative Modeling}.

\bibitem[McAleer et~al., 2018]{mcaleer2018solving}
McAleer, S., Agostinelli, F., Shmakov, A., and Baldi, P. (2018).
\newblock Solving the rubik's cube without human knowledge.
\newblock {\em arXiv preprint arXiv:1805.07470}.

\bibitem[Meister et~al., 2020]{meister2020if}
Meister, C., Cotterell, R., and Vieira, T. (2020).
\newblock If beam search is the answer, what was the question?
\newblock In {\em Proceedings of the 2020 Conference on Empirical Methods in Natural Language Processing (EMNLP)}, pages 2173--2185.

\bibitem[Morozov et~al., 2025]{morozov2025revisiting}
Morozov, N., Maksimov, I., Tiapkin, D., and Samsonov, S. (2025).
\newblock Revisiting non-acyclic gflownets in discrete environments.
\newblock In {\em International Conference on Machine Learning}, pages 44887--44910. PMLR.

\bibitem[Morozov et~al., 2024]{morozov2024improving}
Morozov, N., Tiapkin, D., Samsonov, S., Naumov, A., and Vetrov, D. (2024).
\newblock Improving gflownets with monte carlo tree search.
\newblock In {\em ICML 2024 Workshop on Structured Probabilistic Inference $\&$ Generative Modeling}.

\bibitem[Mulholland, 2016]{mulholland2016permutation}
Mulholland, J. (2016).
\newblock Permutation puzzles: a mathematical perspective.
\newblock {\em Departement Of mathematics Simon fraser University}, 54.

\bibitem[Schrijver et~al., 2003]{schrijver2003combinatorial}
Schrijver, A. et~al. (2003).
\newblock {\em Combinatorial optimization: polyhedra and efficiency}, volume~24.
\newblock Springer.

\bibitem[Sutskever et~al., 2014]{sutskever2014sequence}
Sutskever, I., Vinyals, O., and Le, Q.~V. (2014).
\newblock Sequence to sequence learning with neural networks.
\newblock {\em Advances in neural information processing systems}, 27.

\bibitem[Sutton et~al., 1998]{sutton1998reinforcement}
Sutton, R.~S., Barto, A.~G., et~al. (1998).
\newblock {\em Reinforcement learning: An introduction}, volume~1.
\newblock MIT press Cambridge.

\bibitem[Takano, 2023]{takano2023selfsupervision}
Takano, K. (2023).
\newblock Self-supervision is all you need for solving rubik{\textquoteright}s cube.
\newblock {\em Transactions on Machine Learning Research}.

\bibitem[Tiapkin et~al., 2025]{tiapkin2025gfnx}
Tiapkin, D., Agarkov, A., Morozov, N., Maksimov, I., Tsyganov, A., Gritsaev, T., and Samsonov, S. (2025).
\newblock gfnx: Fast and scalable library for generative flow networks in jax.
\newblock {\em arXiv preprint arXiv:2511.16592}.

\bibitem[Tiapkin et~al., 2024]{tiapkin2024generative}
Tiapkin, D., Morozov, N., Naumov, A., and Vetrov, D.~P. (2024).
\newblock Generative flow networks as entropy-regularized rl.
\newblock In {\em International Conference on Artificial Intelligence and Statistics}, pages 4213--4221. PMLR.

\bibitem[Xiao et~al., 2019]{xiao2019maximum}
Xiao, C., Huang, R., Mei, J., Schuurmans, D., and M{\"u}ller, M. (2019).
\newblock Maximum entropy monte-carlo planning.
\newblock {\em Advances in Neural Information Processing Systems}, 32.

\bibitem[Zhang et~al., 2023]{zhang2023solving}
Zhang, D., Dai, H., Malkin, N., Courville, A.~C., Bengio, Y., and Pan, L. (2023).
\newblock Let the flows tell: Solving graph combinatorial problems with gflownets.
\newblock In {\em Advances in Neural Information Processing Systems}, volume~36, pages 11952--11969.

\bibitem[Zimmermann et~al., 2023]{zimmermann2022variational}
Zimmermann, H., Lindsten, F., van~de Meent, J.-W., and Naesseth, C.~A. (2023).
\newblock A variational perspective on generative flow networks.
\newblock {\em Transactions on Machine Learning Research}.

\end{thebibliography}


\clearpage
\appendix
\thispagestyle{empty}

\onecolumn
\aistatstitle{Learning Shortest Paths with Generative Flow Networks: \\
Supplementary Materials}

\section{Backward Policy Assumptions}\label{app:assumpt}

Consider a random walk on $\cG$ with reversed edges and transition probabilities given by a backward policy $\PB$. Specifically, define a Markov chain $\{X_n\}_{n=0}^\infty$, such that $X_0 = s_f$ a.s. and 
\[
\P[X_t = s \mid X_{t-1} = s'] = \PB(s \mid s')\,.
\]
Also define
\[
\P[X_t = s_0 \mid X_{t-1} = s_0] = 1\,.
\]
Assumption~\ref{pb_assumption} means that the expected length of this walk is finite:
\[
\E[n_{\tau}] = \E[\sum_{t=0}^{\infty} \ind\{X_t \not = s_0 \}] < \infty\,.
\]

This implies that the random walk $X$ terminates in $s_0$ with probability $1$: $\P[\exists t : X_t = s_0 ] = 1$, because otherwise the event $\{ \sum_{t=0}^{\infty} \ind\{X_t \not = s_0 \} = \infty\}$ has non-zero probability, contradicting the finiteness of $\E[n_{\tau}]$. From this, it automatically follows that $\cP(\tau)$ is a correct probability measure over $\cT$ since for any $\tau = (s_0,s_1\ldots,s_{n_{\tau}}, s_f)$ it holds
\[
    \P[(X_{n_{\tau}+1},\ldots,X_{0}) = \tau] = \prod_{t=0}^{n_\tau} \PB(s_{t} \mid s_{t+1}) = \cP(\tau)\,,
\]
and we have
\[
    \sum_{\tau \in \cT} \P[(X_{n_{\tau}+1},\ldots,X_{0}) = \tau] = \P[\exists \tau \in \cT:  (X_{n_{\tau}+1},\ldots,X_{0}) = \tau] = \P[\exists t: X_t = s_0] = 1,
\]
since the events $\{(X_{n_{\tau}+1},\ldots,X_{0}) = \tau\}$ do not intersect for different $\tau$. 

The above reasoning directly follows a part of the proof of Lemma 3.4 of \citet{morozov2025revisiting}. The same Lemma 3.4 shows that the fact $\PB(s \mid s') > 0 \; \forall \tau \in \cT$ implies $\E[n_{\tau}] < \infty$, making it a stronger assumption than Assumption~\ref{pb_assumption}. The main problem of this assumption is that it generally does not hold for $\PB$ with the smallest $\E[n_{\tau}]$, as Theorem~\ref{theorem:smalles_length} suggests, making the theory of \citet{morozov2025revisiting} not directly applicable to the task of learning non-acyclic GFlowNets with the smallest expected trajectory length. However, upon closer inspection of theoretical results of \citet{morozov2025revisiting}, one can note that the only part that actually uses this assumption is Lemma 3.4 itself, and proofs of all other results only require the fact that $\E[n_{\tau}] < \infty$ and the fact that $\cP(\tau)$ is a correct probability measure over $\cT$, thus can be carefully re-written under Assumption~\ref{pb_assumption}.

The same theoretical results under a more general Assumption~\ref{pb_assumption} can also be directly derived from the results which assume $\PB(s \mid s') > 0$. E.g., Remark~\ref{remark_pf} references Proposition 3.8 of~\citet{morozov2025revisiting}, which requires $\PB(s \mid s') > 0$. To apply this result to $\PB$ that can assign zero probability to certain transitions, one can consider a subset $\cT' \subseteq \cT$ consisting of trajectories $\tau'$ with $\cP(\tau') > 0$. This subset induces a subgraph $\cG'$ of $\cG$ that satisfies Assumption~\ref{graph_assumption} since 1) it contains $s_f$ and $s_0$, 2) any state $s$ it contains by construction lies on some trajectory $\tau' \in \cT'$. Then, for any $\PB(s \mid s')$ under Assumption~\ref{pb_assumption}, there exists a unique $\PF$ that induces the same distribution over $\cT'$ by Proposition 3.8 of \citet{morozov2025revisiting} applied to $\cG'$, where $\PB$ assigns non-zero probabilities to all transitions. For transitions not lying in $\cG'$, this $\PF$ is then simply defined to be zero.

\section{Proofs}
\label{app:all_proofs}

\subsection{Proof of Lemma \ref{lemma:lower_bound}}\label{app:lower_bound_proof}
\begin{proof}
First, we apply the tower property of conditional expectation
\begin{equation} \label{eq:trajlen}
    \E[n_\tau] = \E[\E[n_\tau | s_{n_\tau} = x]] = \sum\limits_{x \in \cX} \PB(x | s_f) \E[n_\tau | s_{n_\tau} = x].
\end{equation}
Then, for any $x \in \cX$, since the length of any trajectory from $s_0$ to $x$ by definition is at least $\cl(x)$, we have $\E[n_\tau \mid s_{n_\tau} = x] \ge \cl(x)$, finishing the proof.
\end{proof} 

\subsection{Proof of Lemma \ref{lemma:best_pb_existence}}\label{app:lemma_proofs}

\begin{proof}
Let us define a map $\mathsf{par} \colon \cS \setminus \{s_0, s_f\}\to \cS \setminus \{s_f\}$, where for any $s \in \cS \setminus \{s_0, s_f\}$ the state $\mathsf{par}(s)$ is defined arbitrarily such that $(\mathsf{par}(s) \to s) \in \cE$ and $\cl(s) = \cl(\mathsf{par}(s)) + 1$. Such a state always exists because there always exists a path between $s_0$ and $s$, and precedent state on the corresponding shortest path from $s_0$ to $s$ satisfies $\cl(s) = \cl(\mathsf{par}(s)) + 1$. Then, define $\PB(x \mid s_f) = \cR(x) / \cZ$ for all $x \in \cX$ (to satisfy the reward matching condition), and $\PB(s \mid s') = \ind\{s = \mathsf{par}(s')\}$ for all $s' \in \cS \setminus \{s_f\}$. Therefore, when sampling a trajectory using $\PB$, we first go from $s_f$ to some terminal state $x$, and after that, at each step, we get exactly one edge closer to $s_0$, meaning that a shortest path between $s_0$ and $x$ is always sampled. Thus, for all $x \in \cX$ we have $\E[n_\tau \mid s_{n_\tau} = x] = \cl(x)$. Applying~\eqref{eq:trajlen}, we conclude the proof.
\end{proof}

\begin{remark}
The proof of Lemma~\ref{lemma:best_pb_existence} provided a constructive way to design a backward policy with the smallest $\E[n_\tau]$. It can also be used to constructively provide the unique corresponding forward policy. Once again, for any $s \in S \setminus \{s_0, s_f\}$ we pick an arbitrary state $\mathsf{par}(s)$ such that there is an edge $(\mathsf{par}(s) \to s) \in \cE$ and $\cl(s) = \cl(\mathsf{par}(s)) + 1$, defining $\PB(s \mid s') = \ind\{s = \mathsf{par}(s')\}$ for all $s' \in \cS \setminus \{s_f\}$. If we exclude $s_f$, the structure of $\mathsf{par}(s)$ defines an oriented tree subgraph of $\cG$ rooted in $s_0$, and $\PB$ is degenerate in essence, always choosing the unique parent of a state. In this case, there exists a simple expression for the corresponding forward policy: $\PF(s' \mid s) = V(s') / V(s)$, where $V(s)$ is the sum of GFlowNet rewards over transitions to $s_f$ from vertices in the subtree of $s$, see \cite{bengio2021flow}. And, as a special case, $\PF(s_f \mid s) = \cR(s) / V(s)$.
\end{remark}

\subsection{Proof of Theorem \ref{theorem:smalles_length}}\label{app:proofs}

\begin{proof}
Our proof consists of the two main steps. First, we show that taking only shortest paths guarantees the minimum expected length. Second, we show that minimizing the expected length strictly forbids taking non-shortest paths.

\paragraph{Part 1: If the policy only takes shortest paths, it minimizes the expected length ($\Leftarrow$).}

Assume that if trajectory $\tau$ ending at state $x$ has a non-zero probability ($\cP(\tau) > 0$), its length $n_\tau$ must exactly be equal to the shortest possible path length to that state, denoted as $\cl(x)$. Because the policy satisfies the reward matching condition, we know the probability of terminating at any state $x \in \cX$ is fixed at $\frac{\cR(x)}{\cZ} > 0$. By~\eqref{eq:trajlen}
\[
\E[n_\tau] = \sum_{x \in \cX} \frac{\cR(x)}{\cZ}  \E[n_\tau | s_{n_\tau} =x ] = \sum_{x \in \cX}\frac{\cR(x)}{\cZ} \cl(x)\,.
\]
According to Lemma~\ref{lemma:lower_bound}, this sum represents the absolute theoretical lower bound for the expected length of any valid backward policy. Since our policy achieves this exact lower bound, it successfully minimizes the expected trajectory length.

\paragraph{Part 2: If the policy minimizes expected length, it must only take shortest paths ($\Rightarrow$).}

We will prove this by contradiction. Suppose there is at least one trajectory $\tau$ ending at some terminal state $x'$ that is longer than the shortest path from $s_0$ to $x'$ ($n_\tau > \cl(x')$), and the policy gives this trajectory a non-zero probability ($\cP(\tau) > 0$). Then, we have
\[
\E[n_\tau \mid s_{n_\tau} = x'] > \cl(x').
\]
For all other terminal states $x \in \cX$, the average path length is at least the shortest path length, meaning $\E[n_\tau \mid s_{n_\tau} = x] \ge \cl(x)$. Now, applying~\eqref{eq:trajlen}
\[
\E[n_\tau] = \sum_{x \in \cX} \frac{\cR(x)}{\cZ} \E[n_\tau \mid s_{n_\tau} = x] > \sum_{x \in \cX} \frac{\cR(x)}{\cZ} \cl(x)\,.
\]
However, Lemma~\ref{lemma:best_pb_existence} proves that it is possible to construct a backward policy that achieves the exact theoretical minimum on the right-hand side. Because our policy $\PB$ results in a strictly greater expected trajectory length, it does not minimize $\E[n_\tau]$, which contradicts our initial assumption. 
\end{proof}

\section{Weighted Graphs}\label{app:weighted}

The presented framework can be naturally extended to weighted graphs. Consider the same setup described in Section~\ref{sec:maintheory}. Suppose that now each edge in $\cE$ has a positive weight $w(s \to s')$. Denote the sum of edge weights along a trajectory $\tau$ as $w_\tau$. First, recall that the expected trajectory length coincides with the total normalized state flow: $\E[n_\tau] = \tfrac{1}{\cZ} \sum_{s \in \cS \setminus \{s_0, s_f\}} \cF(s)$. It can also be equivalently rewritten in terms of the total normalized edge flow:
\[
\E[n_\tau] = \sum_{s \to s' \in \cE} \E[\sum_{t = 0}^{n_\tau} \mathbb{I}\{ s_t = s, s_{t+1}=s'\}] 
= \frac{1}{\cZ} \sum_{s \to s' \in \cE} \cF(s \to s').
\]
Following the same logic one can derive an expression for $\E[w_\tau]$:
\[
\E[w_\tau] = \sum_{s \to s' \in \cE} \E[\sum_{t = 0}^{n_\tau} w(s \to s') \mathbb{I}\{ s_t = s, s_{t+1}=s'\}] 
= \frac{1}{\cZ} \sum_{s \to s' \in \cE} w(s \to s') \cF(s \to s').
\]
Thus, expected sum of edge weights coincides with the total normalized edge flow. One can then note that proofs of all the results from Section~\ref{sec:maintheory} can be directly repeated in the weighted case simply by replacing $\E[n_\tau]$ with $\E[w_\tau]$, denoting $\ell(s)$ to be the length of the shortest path from $s_0$ to $s$ in the weighted graph, and taking shortest paths in the weighted graph instead of paths with the smallest number of edges in the constructive proof parts. The reduction from the shortest path problem to non-acyclic GFlowNet training will also be exactly the same as described in Section~\ref{sec:construction}.

Based on this, to train non-acyclic GFlowNets to minimize expected sum of edge weights along a sampled path, one can just replace the state flow regularizer $\lambda \cF(s)$ with a weighted edge flow regularizer $\lambda w(s \to s') \cF(s \to s')$. The reasoning behind the regularized objective will be the same as described in \citet{morozov2025revisiting}: the underlying constraint optimization problem (Equation (11) in \citet{morozov2025revisiting}) will simply include the total weighted edge flow as the objective instead of the total state flow.

\paragraph{Weighted Swap Puzzle Experiment.} Here we provide a small experimental demonstration. We consider the Swap Puzzle from Section~\ref{sec:exp_perms}, and for a transition that swaps $\pi_i$ and $\pi_{i + 1}$ we define the weight as $|\pi_i - \pi_{i+1} - i| + 1$. This form results in a task where two trajectories with the same number of edges may have different sums of weights. We then run the exact same training algorithm presented in our work without any additional tweaking, only replacing the state flow regularizer with the weighted edge flow one. Recall that in our construction the state flow can be expressed as $\cF(s) = 1/\PF(s_f \mid s)$, then one can compute $\cF(s \to s') = \cF(s) \PF(s' \mid s) = \PF(s' \mid s) / \PF(s_f \mid s)$. Then the weighted regularized training objective is:
\begin{align*} \label{eq:wreg_tb_loss}
    \sum\limits_{i = 0}^{N_{\text{max}}}\left\{ \cL_{\mathtt{TB}}(\theta, \tau_{0:i}) + \lambda w(s_i \to s_{i+1})  \frac{ \PF(s_{i+1} \mid s_i, \theta)}{\PF(s_f \mid s_i, \theta)}\right\}.
\end{align*}
The cost of training and the cost of inference with beam search is exactly the same as in the unweighted case. Figure~\ref{fig:weighted} presents the evaluation results on a small environment with $n = 10$. The model was trained with $\lambda = 10^{-3}$ and $N_{max} = 20$. After 50k training steps, the model recovered all exact shortest paths for a uniformly generated test set with the average sum of weights of 66.4 (verified with Dijkstra’s algorithm) with beam search width $W=128$. The experiment shows that our framework can be applied to the weighted case with minimal modification. Although we observed that the model required more test time search budget than for the similar tasks in the unweighted case, it should be possible to improve it with additional tuning of the algorithm. We
consider a deeper examination and evaluation on large-scale weighted graphs an important direction for future research.

\begin{figure}[!t]

    \centering
    \includegraphics[width=0.55\linewidth]{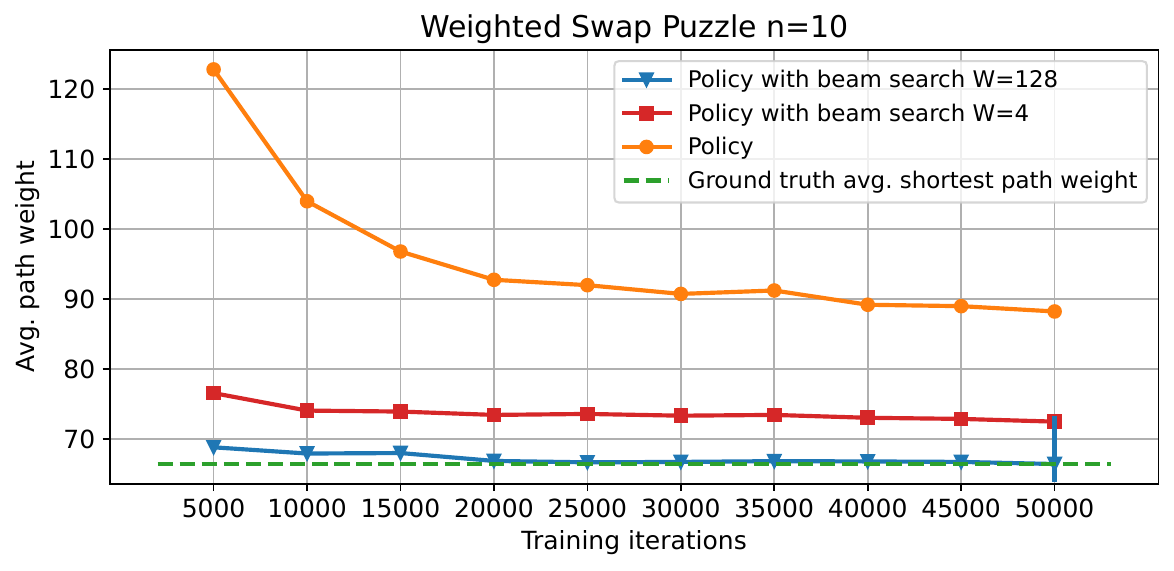}
    \caption{Weighted Swap Puzzle results for $n = 10$.} 
    \label{fig:weighted}
    
\end{figure}

\section{Experimental Details}\label{app:exp_details}

In all experiments, similarly to previous works~\citep{agostinelli2019solving, chervov2025rubik}, we use a variant of residual MLP architecture to parameterize our models, consisting of six $\operatorname{ReLU}(\operatorname{Linear}(\operatorname{LayerNorm}(x)) + x)$ blocks. A difference from previous works is that we replace batch normalization~\citep{ioffe2015batch} with layer normalization~\citep{ba2016layer} to improve training and evaluation stability. $\PF(s' \mid s, \theta)$ and $\PB(s \mid s', \theta)$ share the same backbone, with different linear heads predicting the logits of the forward policy and the logits of the backward policy. Note that the number of forward policy logits is the number of transitions from the state in the initial graph plus one logit corresponding to stop action.

In beam search, we also found an additional heuristic proposed in~\citet{chervov2025rubik} to slightly improve the results, consisting of dropping duplicate states from the current batch at each step of the search, and utilize it in our experiments. It is worth noting that other search algorithms can be potentially applied here, e.g., \citet{morozov2024improving} showed that entropy-regularized Monte Carlo Tree Search~\citep{xiao2019maximum} can be applied to GFlowNets at test-time, posing as an interesting direction for future research.

We use JAX~\citep{jax2018github} for implementing our approach due to its acceleration capabilities through just-in-time (JIT) compilation of the entire training loop, which was shown to significantly speed up GFlowNet training~\citep{tiapkin2025gfnx}. All experiments were carried out on NVIDIA H200 GPU.

\subsection{Swap Puzzle}

We train all models for 100,000 iterations with a batch size of 128 using AdamW optimizer~\citep{loshchilov2017decoupled} with weight decay $10^{-5}$. We set the learning rate to $3 \cdot 10^{-4}$ and use a hidden size of 1024. We found clipping the gradient norm to be helpful to stabilize training, and use a threshold of $100$. Models in Figure~\ref{fig:permsort_main} were trained with regularization coefficient $\lambda = 10^{-3}$ for $n=15$ and $\lambda = 10^{-4}$ for $n=20$. We set $N_{\text{max}} = 50$.

\subsection{Rubik's Cube}

For all Rubik's Cube puzzles, we allow only 90-degree face turns, thus the presented solution lengths are in QTM metric for 2x2x2 and 3x3x3 cases, and in UQTM metric for 4x4x4 case as inner layer turns are allowed. We use the same way to encode states and transitions as permutations as in~\citet{chervov2025rubik}. 

We use 500,000 training iterations with a batch size of 128 for 2x2x2 Rubik's Cube, and 1,000,000 training iterations with a batch size of 2048 for 3x3x3 and 4x4x4 Rubik's Cubes. We set the learning rate to $3 \cdot 10^{-4}$ and use a hidden size of 1024 for 2x2x2 cube and 2048 for 3x3x3 and 4x4x4 cubes. All models are trained using AdamW optimizer~\citep{loshchilov2017decoupled} with weight decay $10^{-5}$. We found clipping the gradient norm to be helpful to stabilize training, and use a threshold of $100$. Models in Table~\ref{cube_table} were trained with regularization coefficient $\lambda = 10^{-2}$ for 2x2x2 cube, $\lambda = 10^{-7}$ for 3x3x3 cube and $\lambda = 10^{-13}$ for 4x4x4 cube. Since flows near the goal state scale with $|V|$ (because we have $\mathcal{F}(s_0) = |V|$), $\lambda$ must decrease as the state space grows, which explains the range of values. We set $N_{\text{max}} = 12$ for 2x2x2 cube,  $N_{\text{max}} = 24$ for 3x3x3 cube and $N_{\text{max}} = 30$ for 4x4x4 cube.  We found these values to lead to efficient training, despite being smaller than longest paths in the respective environments. This further demonstrates generalization capabilities of the proposed approach.

To evaluate against CayleyPy Cube~\citep{chervov2025rubik}, we use the official implementation provided by the authors\footnote{\url{https://github.com/khoruzhii/cayleypy-cube}}. We train all models using the same amount of sampled cube configurations to ensure a fair comparison. We use the test set from~\citet{chervov2025rubik} containing 100 examples for 2x2x2 Rubik's Cube, the test set from~\citet{agostinelli2019solving} containing 1000 examples for 3x3x3 Rubik's Cube, and the test set from~\citet{chervov2025rubik} containing 43 examples for 4x4x4 Rubik's Cube.

\end{document}